\documentclass[conference]{IEEEtran}
\IEEEoverridecommandlockouts
\usepackage{amsmath}
\usepackage{cite}
\usepackage{url}
\usepackage{amsmath}
\usepackage{amsthm}
\theoremstyle{plain}
\newtheorem{lemma}{Lemma}
\usepackage{amsmath,amssymb,amsfonts}
\usepackage{algorithmic}
\usepackage{graphicx}
\usepackage{textcomp}
\usepackage{comment}
\usepackage{xcolor}
\def\BibTeX{{\rm B\kern-.05em{\sc i\kern-.025em b}\kern-.08em
    T\kern-.1667em\lower.7ex\hbox{E}\kern-.125emX}}
\begin{document}

\title{Energy-Efficient SLAM via Joint Design of Sensing, Communication, and Exploration Speed}
\author{Zidong Han\textsuperscript{1,2}, \text{Ruibo Jin}\textsuperscript{2}, \text{Xiaoyang Li}\textsuperscript{3}, \text{Bingpeng Zhou}\textsuperscript{4}, \text{Qinyu Zhang}\textsuperscript{2} \text{ and } \text{Yi Gong}\textsuperscript{1}
\\1 Southern University of Science and Technology, Shenzhen, China
\\ 2 Harbin Institute of Technoglogy (Shenzhen), Shenzhen, China
\\ 3 Shenzhen Research Institute of Big Data, The Chinese University of Hong Kong - Shenzhen, China
\\ 4 The School of Electronics and Communication Engineering, Sun Yat-sen University, Shenzhen, China
\\\{hanzd, gongy, 12110819\}@sustech.edu.cn, zqy@hit.edu.cn, lixiaoyang@sribd.cn, zhoubp3@mail.sysu.edu.cn
}

\maketitle

\begin{abstract}
To support future spatial machine intelligence applications, lifelong \emph{simultaneous localization and mapping} (SLAM) has drawn significant attentions. SLAM is usually realized based on various types of mobile robots performing simultaneous and continuous sensing and communication. This paper focuses on analyzing the energy efficiency of robot operation for lifelong SLAM by jointly considering sensing, communication and mechanical factors. The system model is built based on a robot equipped with a 2D \emph{light detection and ranging} (LiDAR) and an odometry. The cloud point raw data as well as the odometry data are wirelessly transmitted to data center where real-time map reconstruction is realized based on an unsupervised deep learning based method. The sensing duration, transmit power, transmit duration and exploration speed are jointly optimized to minimize the energy consumption. Simulations and experiments demonstrate the performance of our proposed method. 
\end{abstract}


\section{Introduction}
With the rapidly growing need of spatial machine intelligence, generalized \emph{integrated sensing and
communication} (ISAC) technologies have attracted tremendous academic and industrial attentions \cite{liu2022integrated}. As a fundamental technology, \emph{simultaneous localization and mapping} (SLAM) can support a series of spatial intelligence applications, such as auto-driving and unmanned factory, which highly depend on various types of mobile robotic agents \cite{slam-survey1}. Most of the robots are expected to perceive the environment, estimate their system states, interact with edge server and/or other agents via wireless communications, and make decisions autonomously. 

However, real-world environments are usually non-static due to the dynamic changes caused by both ephemeral and persistent objects, which leads to inaccuracy of localization and mapping\cite{LT-mapper}. To address this problem, lifelong SLAM has been proposed to continuously build and maintain the map \cite{lifelong1,lifelong2,lifelong3}. Meanwhile, the sensed data and state information should be delivered on time over time-varying communication links \cite{7393435}. The dynamic properties of both physical environment and communication condition affect the robustness as well as the performance of decision-making and control.

For lifelong SLAM, the long-term energy efficiency and endurance of these robots should be concerned, since they are usually powered by batteries. For the case that the computation in SLAM is offloaded to edge server, the robot generally consists of perception sensors, odometry, electric motors, \textit{micro control unit} (MCU) module, and wireless communication modules. Traditionally, the energy consumptions for robot sensing, movement and wireless data transmission are analyzed independently in its own field, but in fact, they are coupled. The integrated design of sensing, communication, decision and control has drawn great attention in communications society very recently \cite{codesignforSI}. For example, if the sensed data is wirelessly transmitted to edge server for real-time mapping after each sensing period, then for a specific task area, changing movement speed will change the places that communications happen and thus influence the energy consumption for both movement and communication. Meanwhile, the corresponding change of total data amount will also affect the mapping performance. Moreover, regarding to the scalability problem in different scenarios, e.g., different map structures or sizes, the proportions of sensing, movement and communication energy consumption are quite different. How to analyze and optimize the energy efficiency of the robot system operation for completing SLAM tasks is an open issue.

In this paper, we investigate an energy consumption minimization problem for real-time mapping under the timeliness requirement. The system model is established based on an actual mobile robot system equipped with a \emph{two-dimensional light detection and ranging} (2D-LiDAR), an odometry and a wireless communication module. The SLAM task is divided into several operation periods according to the exploration speed and 360-degree sensing duration. The data generated by LiDAR and odometry is wirelessly transmitted to the edge server for map reconstruction which is realized by \textit{deep neural networks} (DNN) based method. The sensing duration, transmit power, transmission duration and exploration speed are jointly designed to minimize the energy consumption. We build a square area and acquire the raw data from LiDAR and odometry to establish a dataset for map learning. We classify it into edge and corner subset in our experiment, and we find it have good scalability.  Our results demonstrate a promising new perspective to design energy-efficient SLAM by comprehensively considering sensing, communication and mechanical factors.

\section{SYSTEM MODEL}

\subsection{System Overview}
The SLAM system considered in this paper is composed of a mobile robot and an edge server (data center), as shown in Fig. \ref{fig:System schematics}. The robot is equipped with a 2D-LiDAR module to perform ranging, an odometry module to measure the state of the robot, a wireless module to transmit the data generated from both LiDAR and odometer to the data center, and an MCU to link and control the above modules. After receiving the data, the real-time mapping, i.e., map reconstruction process, is performed at the data center. 
\begin{figure}[h]
    \centering
    \includegraphics[width=0.8\linewidth]{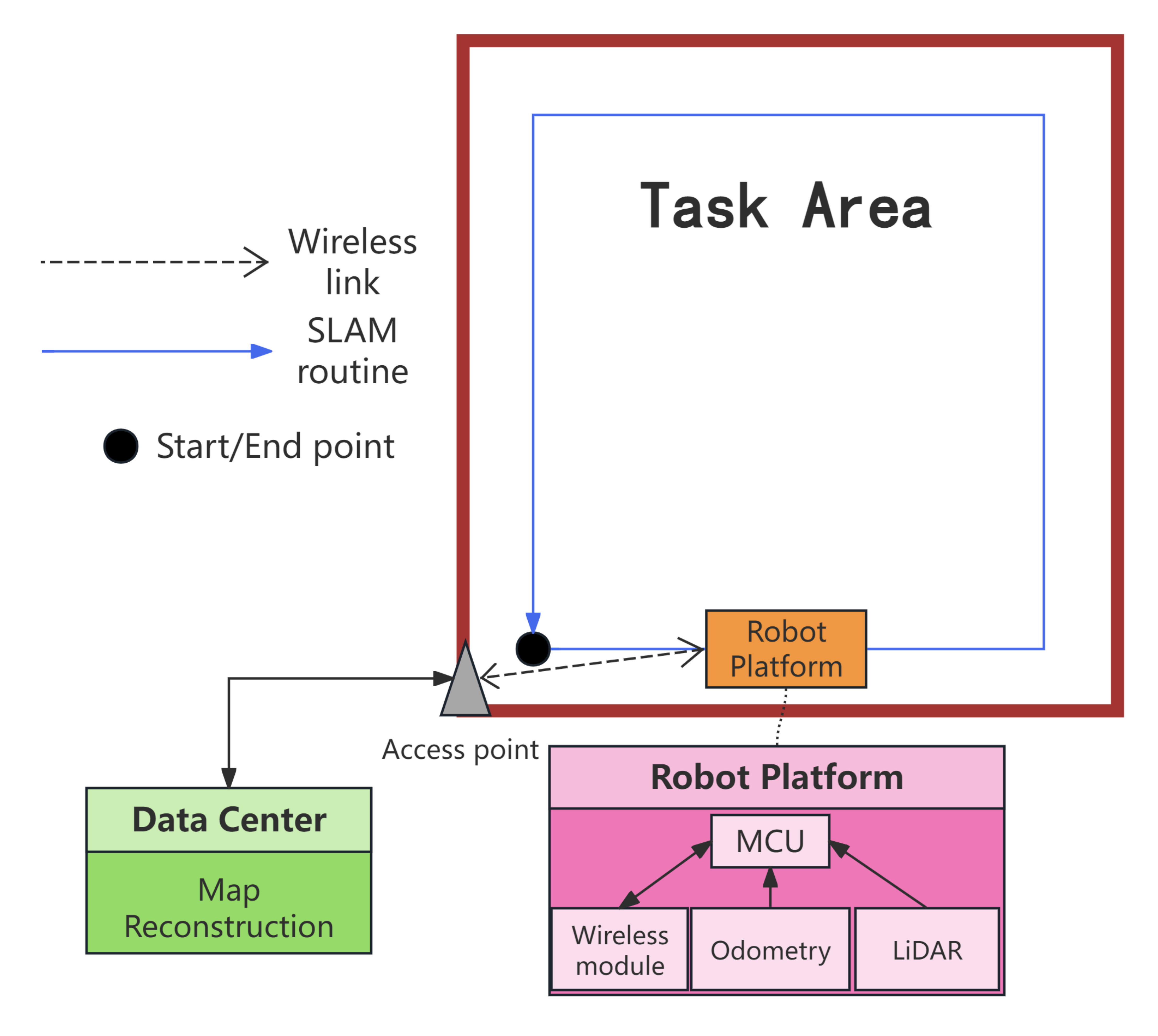}
    \caption{System schematics}
    \label{fig:System schematics}
\end{figure}

The target area under mapping is a square area with a side length of \( L \) m, and an access point linked to the data center is located at one of the four corners. The robot moves along the edge of the area at a distance of \(e\) m in counterclockwise direction\cite{613851}.

The entire SLAM mapping task is divided into several operation periods. For each period, the LiDAR performs a 360-degree sensing process with 1-degree angular resolution. The robot moves and senses simultaneously, and stores the data generated from LiDAR and odometry into cache memory. In next period, the data is transmitted wirelessly from the robot to the data center for map reconstruction. The timing diagram of system operation at robot side is shown in Fig. \ref{fig:Sequence diagram}. It is assumed that the time length for communication is no longer than that for sensing, to ensure the timeliness of data delivery.

Assume the velocity of the robot \textit{v} remains constant in this task, the number of periods $N_m$ is
\begin{equation}
    N_m = \left\lfloor \frac{4(L-2e)}{v t_\text{sens}} \right\rfloor,
\end{equation}
where $t_{sens}$ denotes the 360-degree LiDAR sensing duration, which is equal to the length of each period. It is worth noting that the robot experiences $N_m$  sensing periods, and moves to the end point in the $(N_m+1)$-th period while transmitting the data sensed in the $N_m$-th period.

\begin{figure}[h]
    \centering
    \includegraphics[width=1\linewidth]{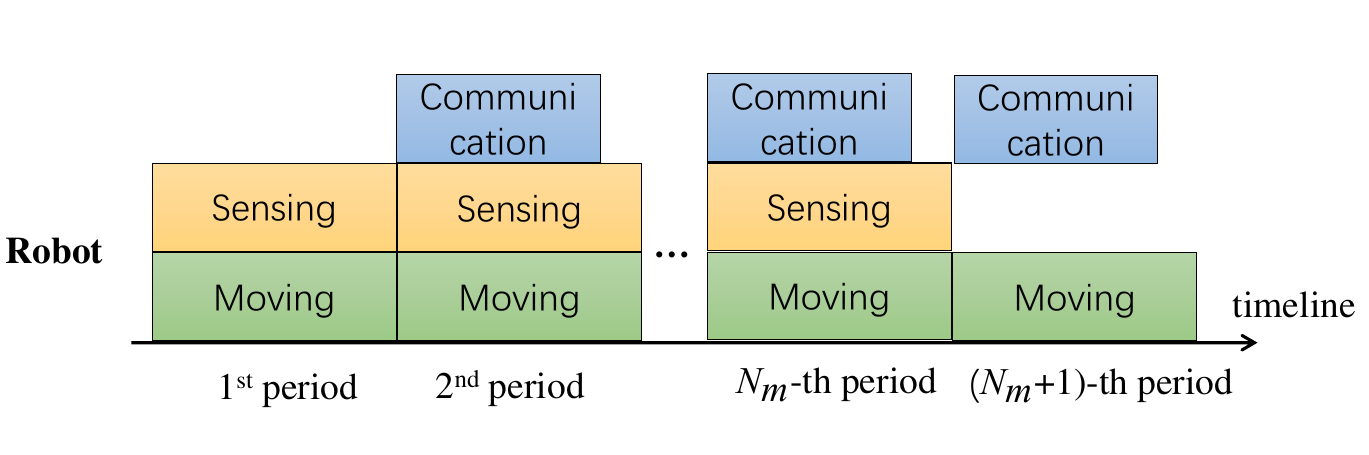}
    \caption{Timing diagram of robot operation}
    \label{fig:Sequence diagram}
\end{figure}

\subsection{Model of sensing process}
The target task area can be transformed into occupancy map \( \mathbf{m}_{0}: \mathbb{R}^{2}\rightarrow \{0,1\} \), where 
\begin{equation}
\mathbf{m}_{0}(x,y)= \begin{cases}
1, & \text{ there is an object at the position \( (x, y) \),}  \\
0, & \text{ there is empty at the position \( (x, y) \)}.  \\
\end{cases}    
\end{equation}
It maps a global coordinate to the corresponding occupancy\cite{8954379}.
For the \( k \)-th period, the robot's state is represented as \( \boldsymbol{V}_k \in \mathbb{R}^{12 \times 1} \), in detail,
\(
\boldsymbol{V}_k = \left(\begin{array}{l}
\overrightarrow{\boldsymbol{p}_k},
\overrightarrow{\boldsymbol{v}_k},
\overrightarrow{\boldsymbol{w}_k},
\overrightarrow{\boldsymbol{a}_k}
\end{array}\right)^T
\).
 \( \overrightarrow{\boldsymbol{p}_{k}} = \left(\begin{array}{l} x_k , y_k , z_k \end{array}\right)^T \) represents the coordinates \( x_k, y_k, z_k\) of the robot position. \( \overrightarrow{\boldsymbol{v}_k} = \left(\begin{array}{l} v_{k, x} , v_{k, y} , v_{k, z},  \end{array}\right)^T \) represents the velocities of the robot in the \( x, y, z \) directions. \( \overrightarrow{\boldsymbol{w}_k} = \left(\begin{array}{l} w_{k,x}, w_{k,y},w_{k,z}\end{array}\right)^T \) represents the angular velocity of the robot in the \( x, y, z \) directions. \( \overrightarrow{\boldsymbol{a}_{k}} = \left(\begin{array}{l} a_{k, x} , a_{k, y} , a_{k, z} \end{array}\right)^T \) represents the accelerations of the robot in the \( x, y, z \) directions.

The data sensed by LiDAR in the \( k \)-th period is  \( \boldsymbol{Z}_k \in \mathbb{R}^{360 \times 1} \), its elements \( Z_{k, \theta}\) are the measured distance between the robot and an object at the angle \( \theta \). So the sensing process is expressed as \cite{bailey2006}
\begin{equation}
\boldsymbol{Z}_k = f\left(\boldsymbol{V}_k, \mathbf{m}_{0}\right) + \boldsymbol{\delta}_k,
\end{equation}
where \( f\left(\boldsymbol{V}_k, \mathbf{m}_{0}\right) \) is the sensing function,   and \( \boldsymbol{\delta}_k \in \mathbb{R}^{360 \times 1} \) represents the sensing error. Assume the sening error is with Gaussian distribution, i.e., \( \boldsymbol{\delta}_k \sim N\left(\mathbf{0}, \boldsymbol{S}_{k}\right) \), where \( \boldsymbol{S}_k = \text{diag}\left(s_{k, 1}, s_{k, 2}, \ldots, s_{k, 360}\right)\in \mathbb{R}^{360 \times 360} \) and \( s_\theta \) is the variance at angle \( \theta \).

The sensed data of the odometry is  \( \boldsymbol{U}_k = \left(\begin{array}{l} a_{k, x}^{\prime} , a_{k, y}^{\prime} , a_{k, z}^{\prime},  w_{k,x}^{\prime}, w_{k,y}^{\prime},w_{k,z}^{\prime}\end{array}\right)^T \in \mathbb{R}^{6 \times 1} \) , where \( a_{k, x}^{\prime}, a_{k, y}^{\prime}, a_{k, z}^{\prime} \) are the measured accelerations of robot in the \( x, y, z \) directions, and \(  w_{k,x}^{\prime}, w_{k,y}^{\prime},w_{k,z}^{\prime} \) are the measured angular velocity of robot in the \( x, y, z \) directions. It can be modeled as 
\begin{equation}
\boldsymbol{U}_k = g\left(\boldsymbol{V}_{k-1}\right) + \boldsymbol{\varepsilon}_k,
\end{equation}
where \( g\left(\boldsymbol{V}_{k-1}\right) \) is the sensing function which measures the acceleration and angular velocity of the robot based on the system state \( \boldsymbol{V}_{k-1} \), \( \boldsymbol{\varepsilon}_k \in \mathbb{R}^{6 \times 1} \) denotes the sensing error under Gaussian distribution, i.e., \( \boldsymbol{\varepsilon}_k \sim N\left(0, \boldsymbol{\Sigma}_k\right) \).

\subsection{Data processing at data center}
\textbf{Map Reconstruction:} We adopt DeepMapping \cite{8954379} for the map reconstruction process, which is based on an unsupervised deep learning method. Fig. \ref{fig:DeepMapping} illustrates the structure of DeepMapping. The estimated point cloud map is updated in real time, denoted by \( \mathbf{M} = \{\mathbf{M}_1, \ldots, \mathbf{M}_{\mathrm{N}_{\mathrm{m}}+1}\} \). The L-Net extracts the features from each local LiDAR data \( \textbf{\textit{Z}}_k \) based on the PointNet \cite{pointnet} and then estimates the pose. KF is Kalman filter to correct the pose estimated by L-net. TF is a transformation module that transforms the local coordinates of \( \textbf{\textit{Z}}_k \) to global coordinates, defined as \( \textbf{\textit{Z}}_{global,k} \). The sampling module is used to samples unoccupied points in the LiDAR scanned areas. Then all the sample points are added to obtain estimated point cloud map \(\textbf{M}_{N_{m+1}}\). M-Net is a binary classification network that uses the estimated point cloud map to predict the occupancy probability for each point being occupied. Those occupancy probabilities are used for computing the loss of map reconstruction.

For M-Net, it is a continuous occupancy map \(m_\phi:\mathbb{R}^{2}\rightarrow [0,1] \) that maps a global coordinate to the corresponding occupancy probability, where \(\phi\) are learnable parameters. Let \( s(\textbf{\textit{Z}}_{global,k}) \) represents the sampling point of the unoccupied area, the loss function is defined as 
\begin{align}
    \mathcal L_{cls}(\phi)=\frac{1}{N_m+1}\sum_{k=1}^{N_m}(B[m_\phi(\textbf{\textit{Z}}_{global,k}),1]+\\B[m_\phi(s(\textbf{\textit{Z}}_{global,k})),0]),
\end{align}
where \(B[m_\phi(\textbf{\textit{Z}}_{global,k}),1]\) denotes the \emph{Binary Cross-Entropy} (BCE) for all points in point cloud \(\textbf{\textit{Z}}_{global,k}\), \(B[m_\phi(s(\textbf{\textit{Z}}_{global,k})),0]\) is the BCE for correspondingly unoccupied locations.

For L-Net, the loss is defined as 
\begin{equation}
    \mathcal L_{ch}(\psi)=\sum_{i=1}^{N_m}\sum_{j\in [i-1,i+1]}D(\textbf{\textit{Z}}_{global,i},\textbf{\textit{Z}}_{global,j}),
\end{equation}
where \(D(Z_{global,i},Z_{global,j})\) represents the Chamfer distance between point cloud \(Z_{global,i}\) and its temporal neighbors \(Z_{global,j}\).

Therefore, the loss function of entire map reconstruction process is defined as 
\begin{equation}
    \mathcal L(\phi,\psi)=\mathcal L_{ch}+\gamma\mathcal L_{cls},
\end{equation}
where \(\gamma\) is a hyperparameter.
\begin{figure}[h]
    \centering
    \includegraphics[width=1.0\linewidth]{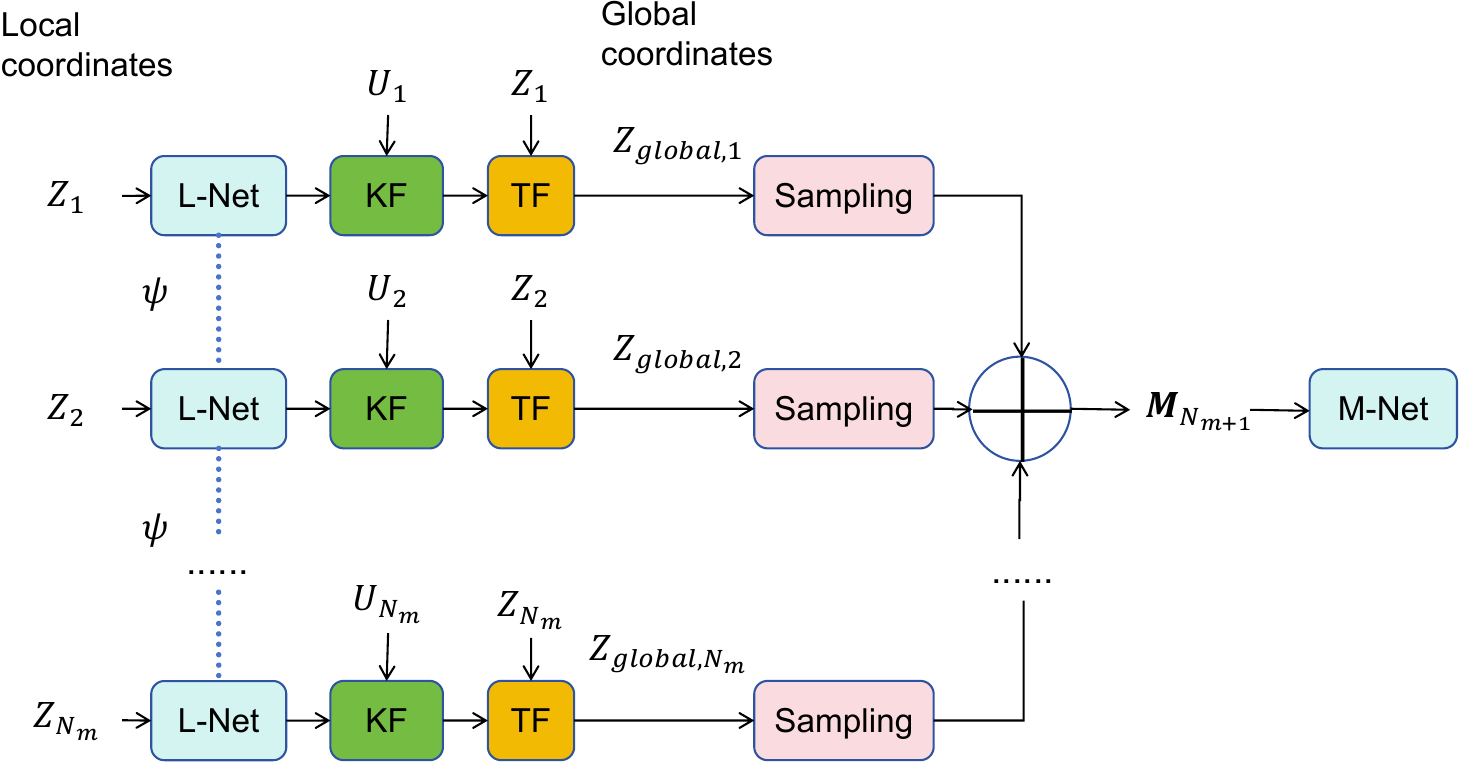}
    \caption{DeepMapping structure}
    \label{fig:DeepMapping}
\end{figure}\\

\subsection{Model of wireless communication process}
Let \( \boldsymbol{b}_k =  (\boldsymbol{Z}_k\thinspace \quad \boldsymbol{U}_k ) ^{\top}\in \mathbb{R}^{363 \times 1} \) denote the required information which is needed for mapping. For each element in \(\boldsymbol{Z}_k\), we assume to use \( a_1 \) bits to express it regarding to data structure. Similarly, \( a_2 \) bits are used for each element in \(\boldsymbol{U}_k\). Thus, let \( \boldsymbol{b}_k^{\prime} \in \mathbb{R}^{(360a_1+6a_2) \times 1} \) denote the complete data sequence in a single period, which is assumed to be coded into several blocks. For simplicity, define $\textbf{\textit{x}}_{k-1}$ as the coded data sequence for the (\textit{k}-1)-th period,
thus the transmission process in the \( k \)-th period can be modeled as
\begin{equation}
    \boldsymbol{y}_k = h_k \boldsymbol{x}_k + \boldsymbol{\omega}_k,
\end{equation}
where \( \boldsymbol{\omega}_k \sim \mathcal{N}(0, \boldsymbol{\sigma}^2_k) \) denotes the Gaussian noise.

Assume the communication process is performed in a multi-path signal propagation environment, there exists a \textit{line-of-sight} (LoS) path \(h_{\text{LoS},k} = \alpha_{0,k} \) and \textit{m} \textit{non-LoS} (NLoS) paths  \( h_{\text{NLoS},k} = \sum_{n=1}^m \alpha_{n,k} e^{j \varphi_{n,k}} \). \( \alpha_{0,k} \) and \( \alpha_{n,k} \) are the amplitude coefficients and \( \varphi_{n,k} \) is the phase shift. Thus, the channel is expressed as
\begin{equation}
    h_k = h_{\text{LoS},k} + h_{\text{NLoS},k}.
\end{equation}
The magnitude of the channel, denoted by \( |h_k| \), follows Rice distribution \cite{rice}. 

Assume the transmit power remains constant within each period and the free-space path loss model is adopted, then the received power at the data center in the \( k \)-th period is given by \cite{friis}:
\begin{equation}
    p_{rx, k}(t)= \frac{p_{tx, k} G_t G_r \lambda^2}{(4 \pi d(t))^2},
\end{equation}
where \( \lambda \) is the signal wavelength, \( G_t \) is the transmit antenna gain of the robot, \( G_r \) is the receive antenna gain of the data center, and \( p_{tx, k} \) is the signal transmit power of the robot. Then the real-time transmission rate \( R_k(t) \) can be expressed as
\begin{equation}
   R_k(t) = B \log_2\left(1 + \frac{ p_{rx, k}(t) |h_k|^2}{\sigma^2_k}\right),
\end{equation}
where $B$ denotes the communication bandwidth. The distance between the robot and the data center \( d(t) \) can be derived as
\begin{equation}
d(t) = \begin{cases}
\scriptstyle \sqrt{(v t + e)^2 + e^2}, & \text{ } 0 \leq t < \frac{L - 2e}{v}, \\
\scriptstyle \sqrt{(v t - L + 3e)^2 + (L - e)^2}, & \text{ } \frac{L - 2e}{v} \leq t < \frac{2(L - 2e)}{v}, \\
\scriptstyle \sqrt{(3L - 5e - v t)^2 + (L - e)^2}, & \text{ } \frac{2(L - 2e)}{v} \leq t < \frac{3(L - 2e)}{v}, \\
\scriptstyle \sqrt{(4L - 7e - v t)^2 + e^2}, & \text{ } \frac{3(L - 2e)}{v} \leq t \leq \frac{4(L - 2e)}{v}.
\end{cases}
\end{equation}
Therefore, the data amount that can be transmitted within the \textit{k}-th period is given by
\begin{align}
\mathrm{I}_k &= \int_{(\mathrm{k} - 1) t_{\text{sens}}}^{(\mathrm{k} - 1) t_{\text{sens}}+t_{\text{comm}}} B \log_2\left(1 + \frac{p_{rx, k}(t) |h_k|^2}{\sigma^2_k}\right) dt,
\label{eq:Total bit}
\end{align}
where $k=2,3,...,N_m+1$ let $t_{\text{comm}}=\rho t_{\text{sens}}$ denote the communication time length with $\rho \in (0,1]$. For simplicity, for the $N_{m+1}$-th period, we assume the communication time length is equal to that in other periods, which does not affect the results since $N_m$ is usually large and the robot is relatively close to the access point.

\subsection{Robot system energy consumption}
The average mechanical power for the robot to move on a flat floor can be modeled as \cite{jetir2204644}

\begin{equation}
    p_{e} = \frac{1}{2} \kappa_1 v^3 + \kappa_2 v,
\end{equation}
where \( \kappa_1 \) denotes the air resistance coefficient (determined by air density, fluid drag coefficient and frontal area), \( \kappa_2 \) denotes the friction coefficient (determined by downforce and rolling friction factor). For other more complex moving resistance model, one can refer to \cite{rolling2} and \cite{rolling3}.
The energy consumption of the LiDAR is assumed to be a constant \( E_{L} \) for each period. 

Therefore, for our entire SLAM mapping task, the total energy consumption of the robot is defined as the sum-energy of the communication energy consumption \(E_{\text{comm}}\), LiDAR energy consumption \(E_{\text{LiDAR}}\) and mechanical  energy consumption \(E_{\text{mech}}\) :
\begin{equation}
    E_{\text{total}} = E_{\text{comm}} + E_{\text{LiDAR}}+ E_{\text{mech}},
    \label{Etotal}
\end{equation}
where
\begin{subequations}
    \begin{align}
    & E_{\text{comm}}= \sum_{k=2}^{N_m+1} p_{tx, k} t_{\text{comm}}\tag{\ref{Etotal}{a}}, \\
    & E_{\text{LiDAR}}= N_m  E_{L}\tag{\ref{Etotal}{b}},\\
    & E_{\text{mech}}= p_{e}(\frac{4(L-2e)}{v})\tag{\ref{Etotal}{c}}.
    \end{align}
\end{subequations}
\setcounter{equation}{27}

%


\section{Problem formulation and solution}
\subsection{Problem Formulation}
Our objective is to minimize the total energy consumption of the robot for completing the SLAM mapping task while meeting the timeliness requirement, with respect to transmit power, communication time length, movement velocity and LiDAR sensing cycle time (also period length here). To this end, the corresponding optimization problem is formulated as
\setcounter{equation}{25}
\begin{subequations}
\begin{align}
\textbf{(P1)} & \min_{\{p_{\text{tx,k}},\rho,t_{\text{sens}},v\}} \quad E_{total}(p_{\text{tx,k}},\rho,t_{\text{sens}},v) \qquad \qquad \notag\\ 
\text{s.t.} &\qquad I_k\geq 360a_1+6a_2,\label{eq:I_k}\\
&\qquad N_m t_{\text{sens}}\leq T_{max}, \\
&\qquad p_{\text{tx,k}} > 0, \\
&\qquad 0<\rho\leq 1,\\
&\qquad N_m \geq N_D,
\end{align}
\end{subequations}
where $T_{max}$ is defined as the maximum allowed time length to complete the task, \(N_D\) is defined as the minimum number of LiDAR sensing cycles for SLAM, $k=1,2, ..., N_m$. (26a) is the requirement for data transmission. (26b) is the task time requirement. (26c) ensures the effectiveness of the power parameter. (26d) is the requirement for communication time length. (26e) is the requirement on number of LiDAR sensing cycles.

\subsection{Solving approach}
In spite of the coupling effect among the system parameters, the relationship between the total energy consumption and the communication parameters is analyzed as shown in the following lemma.

\begin{lemma}
    For problem (P1), given $v$ and $t_{sens}$, \( E_{\text{total}} \) is minimized when
    \(
        \rho^* = 1. 
    \)
    \label{lemma1}
\end{lemma}

\begin{proof}
See Appendix A.
\end{proof}

According to Lemma \ref{lemma1}, the optimal \(p^*_{\text{tx}, k}\) can be achieved based on the following relationship:
\begin{equation}
    \begin{aligned}
        &p^*_{\text{tx}, k}=\\&\left\{p_{\text{tx}, k}\thinspace|\thinspace2^{\frac{N_s(360a_1+6a_2)}{B t_{\text{sens}}}}=\prod_{i=1}^{N_s} \left(1+\frac{p_{\text{tx}, k} \mu_k}{d^2[(k-1+\frac{i}{N_s})t_{\text{sens}}]}\right)\right\}.
        \label{ptxk}
    \end{aligned}
\end{equation}
for \( k=2,3,...,N_m+1\). Then, problem (P1) can be transformed into the following problem: 
\begin{subequations}
\begin{align}
\textbf{(P2)} & \min_{\{t_{\text{sens}},v\}} \quad E_{total}(p^*_{tx,N_m,k},\rho^*,t_{\text{sens}},v) \qquad \qquad \notag\\ 
\text{s.t.}
&\qquad N_m t_{\text{sens}}\leq T_{max},\label{eq:T_max}\\
&\qquad N_m\geq N_D,\label{eq:N_D}
\end{align}
\end{subequations}
where \(p^*_{tx,N_m,k}\) represents the transmission power in the \(k\)-th period and variable \(N_m\) is determined by both \(t_{sens}\) and \(v\). Thus, the total energy can be expressed as
\begin{equation}
\small
\begin{aligned}
&E_{\text{total}}(p^*_{tx,N_m,k},\rho^*,t_{\text{sens}},v) \\&= \left(\frac{1}{2} \kappa_1 v^2 + \kappa_2\right)(4(L-2e))+N_m E_{L}+\sum_{k=2}^{N_m+1} \left( p^*_{tx,N_m,k}t_{\text{sens}} \right).\\
\end{aligned}
\label{eq:total_energy2}
\end{equation}
Given a constant \(v\) and assume \(N_D>>1\), then \(N_m\approx \frac{4(L-2e)}{v t_\text{sens}}\).  
According to (\ref{eq:N_D}), we can have 
\begin{equation} \label{eq:boundtsensv}
    t_{sens}\leq\frac{4(L-2e)}{v N_D}\leq \frac{T_{max}}{N_D}.
\end{equation}
It can be observed from (\ref{eq:total_energy2}) that the energy consumption for robot movement and LiDAR sensing decreases as $t_{sens}$ increases regardless of the physical coefficients. Thus, it is necessary to validate how the total energy changes with $t_{sens}$. As shown in Fig. \ref{fig:evstens}, the total energy as well as the communication energy monotonically decrease with $t_{sens}$, where $t_{sens}$ is taken in a practical range for real-world LiDAR hardware. 
\begin{figure}
    \centering
    \includegraphics[width=1.0\linewidth]{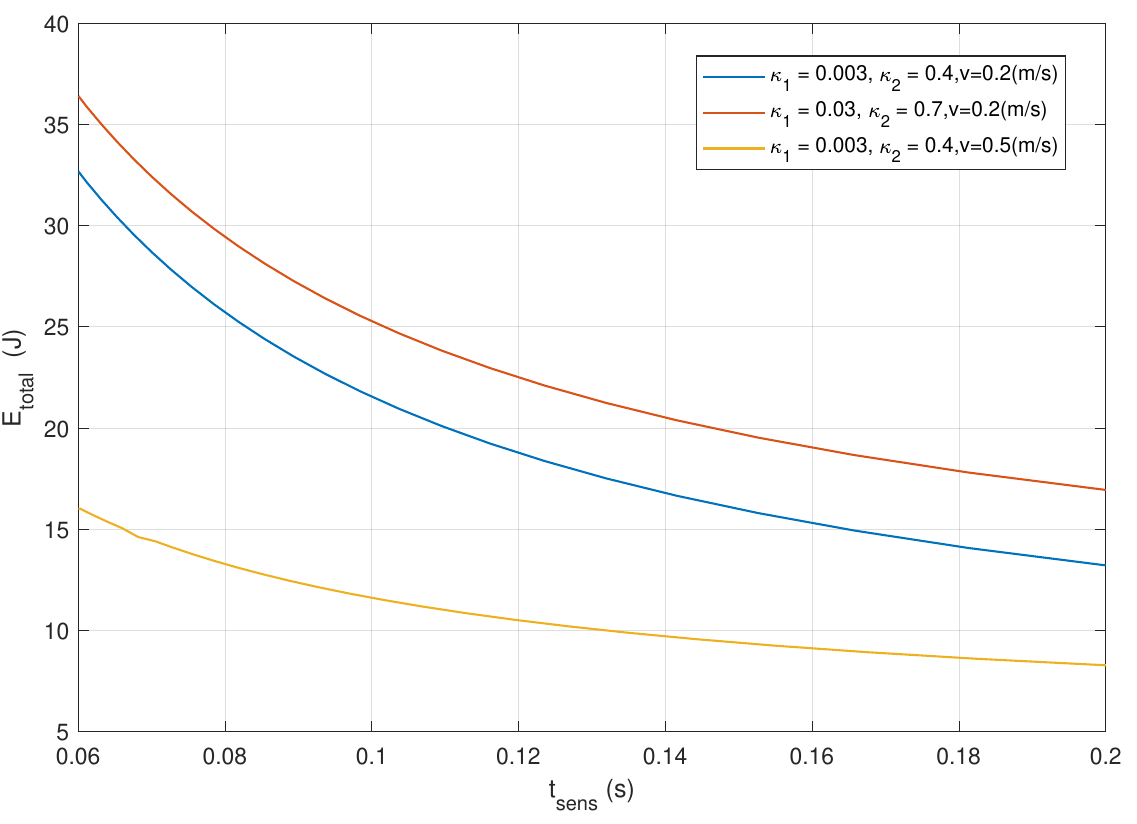}
    \caption{Total energy versus sensing period length}
    \label{fig:evstens}
\end{figure}
Therefore, the optimal period length is achieved as
\begin{equation}
    t_{sens}^*=\frac{4(L-2e)}{v N_D}.
    \label{ts}
\end{equation}
Consequently, \(E_{\text{total}}\) with respect only to speed is expressed as 
\begin{equation}
    \begin{aligned}
        E_{\text{total}}&(v) \\&= \left(\frac{1}{2} \kappa_1 v^2 + \kappa_2\right)(4(L-2e))+ N_D E_{L}
 \\&\quad+ \frac{4(L-2e)}{v N_D}\sum_{k=2}^{N_D+1}p^*_{tx,N_D,k}.
    \end{aligned}
\end{equation}
Note that $p^*_{tx,N_D,k}$ is still related to $v$ and is very hard to get closed-form solution. Since the real-time communication distance in each period has a maximum value $d_{k,max}$, we analyze the upper bound of $E_{total}$. When $N_D$ is relatively large, the upper bound is close to the real value.

Based on Lemma 1, the corresponding transmit powers in each period are achieved by
\begin{equation}
    \begin{aligned}
        &p_{tx,N_D,k,max}\\&=\left\{p_{tx,N_D, k}\thinspace|\thinspace2^{\frac{N_s(360a_1+6a_2)N_D}{4B(L-2e)}}=\prod_{i=1}^{N_s} \left(1+\frac{p_{\text{tx},N_D, k} \mu_k}{d_{k,max}^{2}}\right)\right\}\\
        &=\left(2^{\xi v}-1\right)\frac{d^{2}_{k,max}}{\mu_k},
    \end{aligned}
\end{equation}
where
\begin{equation}
    \xi=\frac{(360a_1+6a_2)N_D}{4B(L-2e)}.
\end{equation}
 Thus, the upper bound of total energy is derived as 
\begin{equation}
    \begin{aligned}
        E_{total,up}&(v) \\&= \left(\frac{1}{2} \kappa_1 v^2 + \kappa_2\right)(4(L-2e))+ N_D E_{L}
 \\&\quad+ \frac{4(L-2e)}{v N_D}\sum_{k=2}^{N_D+1}\left(2^{\xi v}-1\right)\frac{d^{2}_{k,max}}{\mu_k}.
    \end{aligned}
\end{equation}
The derivative of \(E_{total,up}\) is 
\begin{equation}
    \begin{aligned}
    &\frac{d E_{total,up}(v) }{d v}=\\&4(L-2e)\left[\frac{v^3 \kappa_1 N_D + \left( (v \xi \ln 2 - 1) 2^{\xi v} + 1 \right) \sum_{k=2}^{N_D+1} \frac{d_{k, max}^2}{\mu_k} }{N_D v^2} \right].
    \end{aligned}
\end{equation}
Let 
\begin{equation}
    \eta(v)=v^3 \kappa_1 N_D + \left( (v \xi \ln 2 - 1) 2^{\xi v} + 1 \right) \sum_{k=2}^{N_D+1} \frac{d_{k, max}^2}{\mu_k} .
\end{equation}
It can be observed that, \(\eta(v)\) is a monotonically increasing function of $v$ for $v>0$, and \(\eta(v) = 0^+|_{v=0^+}\) Thus,
\begin{equation}
    \frac{d E_{total,up}(v) }{d v}>0.
\end{equation}
According to (\ref{eq:T_max}) and (\ref{ts}), we have a requirement on speed:
\begin{equation}
    v\geq \frac{4(L-2e)}{T_{max}}.
\end{equation}
Therefore, the optimal \(v\) that minimizes \(E_{total,up}(v)\) is
\begin{equation}
    v^*=\frac{4(L-2e)}{T_{max}},
\end{equation}
and eventually
\begin{equation}
    \begin{aligned}
        &E_{total,up,min}= \\& \left(\frac{1}{2} \kappa_1 \frac{4(L-2e)}{T_{max}}^2 + \kappa_2\right)(4(L-2e))+ N_D E_{L}
 \\&\quad+ \frac{T_{max}}{N_D}\sum_{k=2}^{N_D+1}\left(2^{ \frac{4 \xi (L-2e)}{T_{max}} }-1\right)\frac{d^{2}_{k,max}}{\mu_k}.
    \end{aligned}
\end{equation}

\section{Experiment and Numerical Results}
This section provides the numerical results in which the map reconstruction is trained based on SLAM experiments.  

\textbf{Dataset.} The map reconstruction process is realized based on a dataset built by ourself. A microROS robot, which is equipped with a 2D LiDAR of 360-degree \emph{field of view} (FOV) and an odometry, is used to obtain the cloud point raw data. The data is transmitted back via a WiFi link. As shown in Fig. X, we acquire the LiDAR raw data by creating a 2.25x2.25 $m^2$ square area by fence and generally classify it into two subset, i.e., edge and corner. Based on the preset trajectory of robot movement, we totally accumulate cloud point data of 20 cycles in that area, i.e., about 180 m edge and 80 corners. We randomly select part of this dataset and create up to 20x20 $m^2$ square area for map learning and energy consumption evaluation. The deep learning network parameters are optimized and trained on an Nvidia GTX3080.  

\textbf{System parameters.} For communication, the bandwidth is set as $B=10$ MHz. The wavelength $\lambda$ is 0.125 m for 2.4 GHz. The average noise power spectral density over this band is -110 dBm/Hz. The channel under multipath is modeled as Rice distribution. For 2D LiDAR sensing, the number of bits to represent cloud point data and odometry data is $a_1=a_2=64$. The time length of a single sensing cycle is considered within 0.06 s to 0.2 s, i.e., about 5  to 16 Hz scan frequency. For each sensing cycle, the energy $E_L$ is set to be 0.025 J. For robot movement, $\kappa_1$ and $\kappa_2$ are set to be 0.003 and 0.4, respectively. The area size $L$ is considered from 2 to 20 m and $e$ is 0.45 m. The allowed time for SLAM task is set as $T_{max}=40 
$ s. The requirement on the number of periods is $N_D=400$.

\begin{figure}
    \centering
    \includegraphics[width=1.0\linewidth]{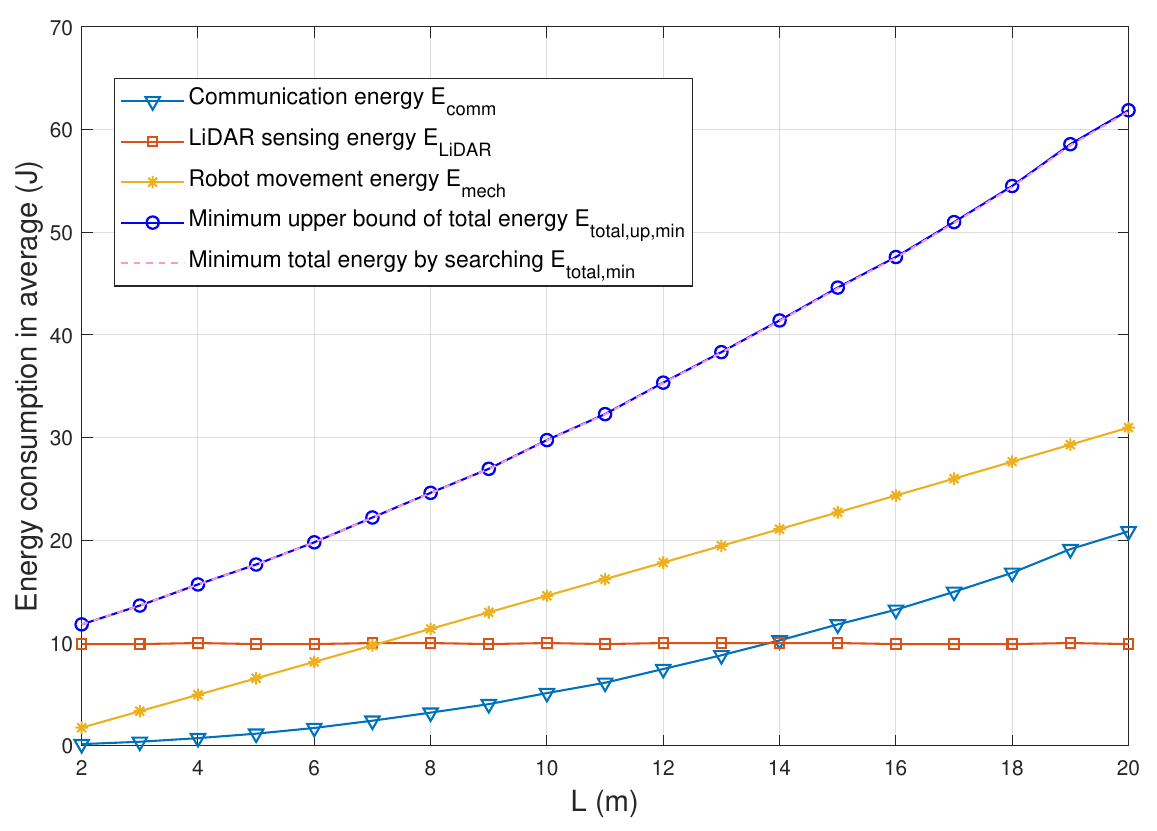}
    \caption{Energy consumption versus size of task area}
    \label{fig:energytotalvsL}
\end{figure}

Fig. \ref{fig:energytotalvsL} illustrates the energy consumption of the entire robot system and the corresponding sensing, communication, movement part with the increasing area size. The edge length is considered ranging from 2 to 20 m. It can be observed that the communication energy exponentially increases according to the wave propagation property. The energy for robot movement nearly linearly increases, since the moving resistance coefficient is small in our simulation for the small robot and the movement speed requirement is also small. The energy for LiDAR sensing remains constant based on our system setting. So, for relatively small area under lifelong SLAM, we need to pay more attention to the energy-efficiency of sensing and movement, but this is no longer true as the area becomes larger. 

Fig. \ref{fig:maprecons} depicts a reconstructed map based on one-cycle run. In real-world environment, the SLAM routine is usually not good with the pre-designed due to many factors such as the odometry errors that influence both mapping and communication operations. So, more advanced pose estimation and localization algorithms to adapt to various real-world odometries are needed. High-precise estimation and localization will also bring great benefit for communication scheme design.


\begin{figure}
    \centering
    \includegraphics[width=1.0\linewidth]{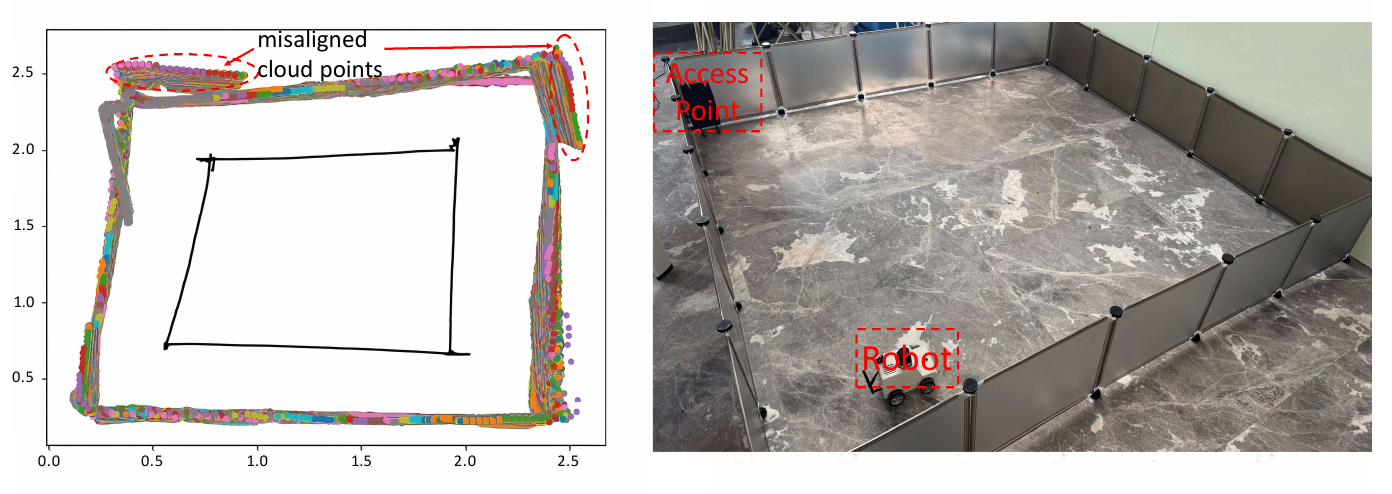}
    \caption{Map reconstruction of a 2.25x2.25 $m^2$ area (run for one cycle)}
    \label{fig:maprecons}
\end{figure}


\section{CONCLUSION}
This paper focuses on the energy consumption problem to support energy-efficient lifelong SLAM and the corresponding potential spatial machine intelligence applications. The system model is built based on a robot performing SLAM via 2D LiDAR sensing and deep learning based map reconstruction method. It is found that the energy proportions of sensing, communication and movement are quite different with increasing area sizes. Our results demonstrate a promising new perspective to analyze the energy efficiency for lifelong SLAM, leading to many interesting research directions such as multi-agent resource allocation, more efficient design of channel estimation, re-localization and loop closure.


\appendix
\subsection{Proof of Lemma 1}
\begin{proof}
Give a constant $v$ and $t_{sens}$, $E_{\text{total}}$ is simplified as a function with respect to $\rho$ and $p_{tx,k}$ that are only relevant with $E_{comm}$.

Due to the difficulty to achieve the closed-form expression of the continuous-time integration in (\ref{eq:Total bit}), for the interval \([(k-1)t_{sens}, (k-1)t_{sens}]\), it can be discretized into \(N_s\) equal-length sub-intervals and the \(i_{th}\) sub-interval is \([(k-1+\frac{(i-1)}{N_s})t_{sens}, (k-1+\frac{i}{N_s})t_{sens}]\) where \(i = 1,2,\cdots,N_s\). For the \(i_{th}\) sub-interval,  we defined \(\rho_i\) as the ratio between communication time length to the sub-interval length, where \(\rho_i \in (0,1]\). Thus, we have
    \begin{equation}
        \sum_{i=1}^{N_s}\rho_i \frac{t_{sens}}{N_s}=\rho t_{sens}.
        \label{p}
    \end{equation}
\(I_k\) can be approximately expressed in the following form:
    \begin{equation}
         I_k\approx \sum_{i=1}^{N_s} B \log_2 \left( 1 + \frac{p_{\text{tx}, k} \mu_k}{d^2[(k-1+\frac{i}{N_s})t_{sens})]} \right)\frac{1}{N_s}\rho_i t_{\text{sens}},
         \label{eq:IK1}
    \end{equation}
    where $N_s$ should be chosen as a relatively large positive integer, and let
    \begin{equation}\label{eq:optimization}
     \mu_k= \frac{ G_t G_r \lambda^2 |h_k|^2}{\sigma^2_k(4 \pi)^2}.
    \end{equation}
    For the \({i_{th}}\) sub-interval, let
     \begin{equation}
         I_{k,i}= B \log_2 \left( 1 + \frac{p_{\text{tx}, k} \mu_k}{d^2[(k-1+\frac{i}{N_s})t_{\text{sens}}]} \right)\frac{1}{N_s}\rho_i t_{\text{sens}},
         \label{eq:IK2}
    \end{equation}
    for $i=1,2,..., N_s$, denotes the data amount that can be transmitted within the \({i_{th}}\) sub-interval. \(I_{k,i}\) is a parameter that is related to \(I_k\) and independent of \(\rho_i\)
    
    Based on the above derivations, \(E_{comm}\) can be expressed as
    \begin{equation}
        \begin{aligned}
             E_{\text{comm}}&\approx \sum_{k=1}^{N_m+1} p_{tx, k} \sum_{i=1}^{N_S}\frac{\rho_i t_{sens}}{N_s}\\
             &\approx\sum_{k=1}^{N_m+1}\sum_{i=1}^{N_s}\frac{\beta_{1,k,i}p_{\text{tx}, k}}{\log_2 \left( 1 +\beta_{2,k,i} p_{\text{tx}, k} \right) }.
        \end{aligned}
    \end{equation}
    where \(\beta_{1,k,i}=\frac{I_{k,i}}{B} >0\), \(\beta_{2,k,i}=\frac{\mu_k}{d^2[(k-1+\frac{i}{N_s})t_{\text{sens}}]}>0\).
    The partial derivative of \(E_{comm}\) with respect to \( p_{\text{tx}, k} \) is given by
    \small{
    \begin{equation}
        \begin{aligned}
            &\frac{\partial E_{comm}}{\partial p_{\text{tx}, k}} =\\
            &\ln{2}\sum_{k=1}^{N_m}\sum_{i=1}^{N_s}\frac{\beta_{1,k,i}[(1+ p_{\text{tx}, k}\beta_{2,k,i})\ln{(1+ p_{\text{tx}, k}\beta_{2,i})}-p_{\text{tx}, k}\beta_{2,k,i}]}{(1+ p_{\text{tx}, k}\beta_{2,k,i})\ln^2{(1+ p_{\text{tx}, k}\beta_{2,i})}}.
        \end{aligned}
    \end{equation}}
    Let 
    \begin{equation}
        \begin{aligned}
        &\zeta_k(p_{\text{tx}})=\\&\beta_{1,k,i}[(1+ p_{\text{tx}, k}\beta_{2,k,i})\ln{(1+ p_{\text{tx}, k}\beta_{2,k,i})}-p_{\text{tx}, k}\beta_{2,k,i}].
        \end{aligned}
    \end{equation}
    The derivative of \( \zeta_k (p_{\text{tx}}) \) with respect to \( p_{\text{tx}, k} \) is given by
    \begin{equation}
         \frac{d \zeta_k (p_{\text{tx}})}{d p_{\text{tx}, k}}=\beta_{1,k,i}\beta_{2,k,i}\ln{(1+p_{\text{tx}, k}\beta_{2,k,i})}>0,
    \end{equation}
    It can be verified that \(\zeta_k (p_{\text{tx}})> 0\). Therefore,
    \begin{equation}
        \frac{\partial E_{comm}}{\partial p_{\text{tx}, k}} > 0.
    \end{equation}
    According to formula (\ref{eq:Total bit}) and (\ref{eq:I_k}),we have
    \begin{equation}
        360a_1+6a_2= \int_{(\mathrm{k} - 1) t_{\text{sens}}}^{(\mathrm{k} - 1+\rho) t_{\text{sens}}} B \log_2\left(1 + \frac{p_{\text{tx}, k}\mu_k}{d^2(t)}\right) dt.
        \label{Ptx}
    \end{equation}
    So, according to the implicit function derivative rule \cite{derivative}, we have
    \begin{equation}
        \frac{\partial p_{\text{tx}, k}}{\partial \rho}=-\frac{\ln{\left(1 + \frac{p_{\text{tx}, k}\mu_k}{d^2[(k-1+\rho)t_{\text{sens}}]}\right)t_{sens}}}{\int_{(\mathrm{k} - 1) t_{\text{sens}}}^{(\mathrm{k} - 1+\rho) t_{\text{sens}}} \frac{\mu_k}{d^2(t)+p_{\text{tx}, k}\mu_k} dt}<0.
    \end{equation}
Therefore, regarding to the derivative of $E_{comm}$ with respect to $\rho$, i.e.,
    \begin{equation}
        \frac{d E_{comm}}{d \rho} =\frac{\partial E_{comm}}{\partial p_{\text{tx}, k}} \frac{d p_{\text{tx}, k}}{d \rho}+\sum_{i=1}^{N_s}\frac{\partial E_{comm}}{\partial I_{k,i}} \frac{d I_{k,i}}{d \rho},
    \end{equation}
    since \(I_{k,i}\) and \(\rho\) are independent, we have \(\frac{d I_{k,i}}{d \rho}=0\). Consequently, we have
    \begin{equation}
        \frac{d E_{comm}}{d \rho}<0.
    \end{equation}
     And the optimal \(\rho^*\) is 
    \begin{equation}
        \rho^*=1
    \end{equation}
    
    End Proof.
\end{proof}


\begin{thebibliography}{00}

\bibitem{liu2022integrated}
F. Liu, Y. Cui, C. Masouros, J. Xu, et al., “Integrated sensing and communications: Toward dual-functional wireless networks for 6G and beyond,” \textit{IEEE J. Sel. Areas Commun.}, vol. 40, no. 6, pp. 1728–1767, Mar. 2022.

\bibitem{slam-survey1}
C. Cadena, L. Carlone, H. Carrillo, Y. Latif, D. Scaramuzza and J. Neira, “Past, present, and future of simultaneous localization and mapping: Toward the robust-perception age,” \textit{IEEE Trans. Robot.}, vol. 32, no. 6, pp. 1309-1332, Dec. 2016.

\bibitem{LT-mapper}
G. Kim and A. Kim, “LT-mapper: A modular framework for LiDAR-based lifelong mapping,” in \textit{Proc. International Conference on Robotics and Automation}, PA, USA, 2022.

\bibitem{lifelong1}
F. Pomerleau, P. Krüsi, F. Colas, P. Furgale and R. Siegwart, “Long-term 3D map maintenance in dynamic environments,” in \textit{Proc. IEEE International Conference on Robotics and Automation}, Hong Kong, China, 2014.

\bibitem{lifelong2}
M. Zhao, X. Guo, L. Song, B. Qin, X. Shi, G. H. Lee and G. Sun, “A general framework for lifelong localization and mapping in changing environment,” in \textit{Proc. IEEE/RSJ International Conference on Intelligent Robots and Systems}, Prague, Czech Republic, 2021.

\bibitem{lifelong3}
S. Zhu, X. Zhang, S. Guo, J. Li and H. Liu, “Lifelong localization in semi-dynamic environment,” in \textit{Proc. IEEE International Conference on Robotics and Automation}, Xi'an, China, 2021.

\bibitem{7393435}
J. Zhang, T. Q. Duong, A. Marshall and R. Woods, “Key Generation From Wireless Channels: A Review,” \textit{IEEE Access}, vol. 4, pp. 614-626, 2016.

\bibitem{codesignforSI}
Z. Han, X. Li, Z. Zhou, K. Huang, Y. Gong and Q. Zhang, “Wireless communication and control co-design for system identification,” \textit{IEEE Trans. Wireless Commun.}, vol. 23, no. 5, pp. 4114-4126, May 2024.

\bibitem{bailey2006}
T. Bailey, J. Nieto, J. Guivant, M. Stevens, and E. Nebot, “Consistency of the EKF-SLAM Algorithm,” in \textit{Proc. IEEE/RSJ International Conference on Intelligent Robots and Systems}, Beijing, China, 2006.

\bibitem{pointnet}
R. Q. Charles, H. Su, M. Kaichun and L. J. Guibas, “PointNet: Deep Learning on Point Sets for 3D Classification and Segmentation,” in \textit{Proc. IEEE Conference on Computer Vision and Pattern Recognition}, Honolulu, HI, USA, 2017.

\bibitem{613851}
B. Yamauchi, “A frontier-based approach for autonomous exploration,” in \textit{Proc. IEEE International Symposium on Computational Intelligence in Robotics and Automation}, 1997.

\bibitem{b2}
T. D. Barfoot, \textit{State Estimation for Robotics}. Cambridge, UK: Cambridge University Press, 2017.

\bibitem{8954379}
Li Ding and Chen Feng, “DeepMapping: Unsupervised map estimation from multiple point clouds,” in \textit{Proc. IEEE/CVF Conference on Computer Vision and Pattern Recognition}, 2019.

\bibitem{rice}
C. Tepedelenlioglu, A. Abdi and G. B. Giannakis, "The Ricean K factor: Estimation and performance analysis," IEEE Transactions on Wireless Communications, vol. 2, no. 4, pp. 799-810, Jul 2003. 

\bibitem{friis}
A. Goldsmith, \textit{Wireless communication}. Cambridge University Press, 2005.

\bibitem{jetir2204644}
H. Barman, S. Dey, A. Maity, A. R. Goswami, S. Duari, P. Biswas and A. Chatterjee, “An investigation: Effect of resistance for different passenger cars,” \textit{International Journal of Emerging Technologies and Innovative Research}, vol. 9, no. 4, pp. 290--297, Apr. 2022.

\bibitem{rolling2}
E. Evans and P. Zemroch, “Measurement of the aerodynamic and rolling resistances of road tanker vehicles from coast-down tests.” \textit{Inst. Mech. Eng. Part D Transp. Eng.}, vol. 198, no. 3, pp. 211–218, 1984. 

\bibitem{rolling3}
L. G. Andersen, J. K. Larsen, E. S. Fraser, B. Schmidt and J. C. Dyre, “Rolling resistance measurement and model development,” \textit{Journal of Transportation Engineering}, vol. 141, no. 2, pp. 1-10, Feb 2015.

\bibitem{derivative}
A. Banner, \textit{The Calculus Lifesaver: All the Tools You Need to Excel at Calculus}. Princeton University Press, 2007.

\end{thebibliography}
\end{document}